%% file: paper.tex
\newif\ifarxiv
\arxivtrue
\newcommand{\splittableref}[2]{%
\ifarxiv\ref{#1}\else#2\fi%
}

\newcommand{\resizearxiv}[1]{%
\ifarxiv#1\else\resizebox{\textwidth}{!}{#1}\fi%
}
\ifarxiv
\documentclass{article}
\usepackage{arxiv}
\usepackage[numbers,compress]{natbib}
\usepackage{amsthm}
\usepackage{booktabs}

\theoremstyle{plain}
\newtheorem{theorem}{Theorem}

\newtheorem{lemma}{Lemma}
\newtheorem{corollary}{Corollary}
\theoremstyle{definition}
\newtheorem{definition}{Definition}
\theoremstyle{remark}
\newtheorem{remark}{Remark}
\newtheorem{example}{Example}
\newtheorem{question}{Question}
\else
\documentclass[runningheads]{llncs}
\usepackage[T1]{fontenc}
\usepackage{graphicx}
\usepackage{booktabs}
\usepackage[misc]{ifsym}

\fi

\usepackage[utf8]{inputenc} 
\usepackage[hidelinks]{hyperref}       
\usepackage{url}            
\usepackage{nicefrac}       
\usepackage{microtype}      
\usepackage{xcolor}         
\usepackage{balance}
\usepackage{tikz}
\usetikzlibrary{patterns}
\usetikzlibrary{automata, positioning}
\usepackage{textcomp, gensymb}
\usepackage{paralist, tabularx}
\usepackage{amsmath,amssymb,amsfonts}
\usepackage{acronym}

\definecolor{mydarkblue}{rgb}{0,0.08,0.45}
\hypersetup{ %
pdftitle={},
pdfkeywords={},
pdfborder=0 0 0,
pdfpagemode=UseNone,
colorlinks=true,
linkcolor=mydarkblue,
citecolor=mydarkblue,
filecolor=mydarkblue,
urlcolor=mydarkblue,
}
\ifarxiv\else
\newcommand{\citep}{\cite}
\newcommand{\citet}{\cite}
\fi

\input{macros}

\begin{document}

\title{Viability of Future Actions: \texorpdfstring{\\}{}
Robust Safety in Reinforcement Learning via Entropy Regularization
}
\titlerunning{Robust safety in RL via Entropy Regularization}
\ifarxiv
\author{
    Pierre-François Massiani\textsuperscript{1,*}
    \And
    Alexander von Rohr\textsuperscript{1,2,*}\\
    \And
    Lukas Haverbeck\textsuperscript{1}
    \And
    Sebastian Trimpe\textsuperscript{1}
    \AND \phantom{a}\\\textsuperscript{*} Equal contribution\\
        \textsuperscript{1} Institute for Data Science in Mechanical Engineering, RWTH Aachen University, Germany\\
        \textsuperscript{2} Learning Systems and Robotics Lab, Technical University of Munich, Germany
        \\\texttt{\{massiani,lukas.haverbeck,trimpe\}@dsme.rwth-aachen.de}\\
        \texttt{alex.von.rohr@tum.de}
}
\fi
\maketitle
\begingroup
\renewcommand{\thefootnote}{}
\footnotetext{Accepted for publication at ECML-PKDD 2025.}
\endgroup

\begin{abstract}
Despite the many recent advances in reinforcement learning (RL), the question of learning policies that robustly satisfy state constraints under unknown disturbances remains open.
In this paper, we offer a new perspective on achieving robust safety by analyzing the interplay between two well-established techniques in model-free RL: entropy regularization, and constraints penalization. 
We reveal empirically that entropy regularization in constrained RL inherently biases learning toward maximizing the number of future viable actions, thereby promoting constraints satisfaction robust to action noise.
Furthermore, we show that by relaxing strict safety constraints through penalties, the constrained RL problem can be approximated arbitrarily closely by an unconstrained one and thus solved using standard model-free RL.
This reformulation preserves both safety and optimality while empirically improving resilience to disturbances.
Our results indicate that the connection between entropy regularization and robustness is a promising avenue for further empirical and theoretical investigation, as it enables robust safety in RL through simple reward shaping.

\end{abstract}
\section{Introduction}

Safety is the ability of a policy to keep the system away from a failure set of undesirable states.
Robustness extends the notion to adversarial or noisy settings; robust policies remain outside of the failure set in spite of the noise or adversary.
While robust \ac{RL} may be formulated as a constrained optimization problem\,\citep{moos2022robust}, there is a strong appeal in achieving robustly safe policies through reward shaping alone, given the numerous algorithms available for unconstrained \ac{RL}.
The purpose of this work is to reveal how robustly safe policies arise naturally from two common practices in \ac{RL}; namely, maximum-entropy \ac{RL} \citep{HZA2018} and failure penalization \citep{MHST2022}.
Our results support that the maximum-entropy \ac{RL} objective together with failure penalties enable safe operation at testing under action noise stronger than that seen at training; a property we call \emph{robustness}.
\par

\begin{figure*}[t]
    \centering
    \resizearxiv{\includegraphics{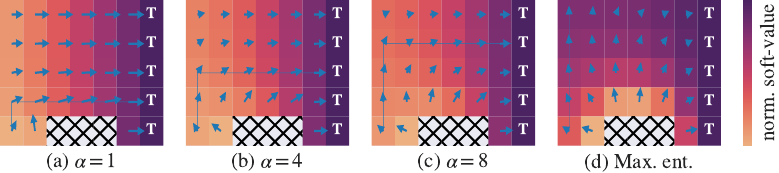}}
    \caption{\textbf{Fenced cliff --- Robustness as a function of $\alpha$:} An entropy-regularized policy avoids states with fewer actions available (d).
    The degree is controlled by the temperature parameter $\alpha$. As it increases (a--c), the policy moves away from the constraints, getting more robust but taking longer to reach the target.
    The mode of the policy is shown as a thin blue line.}
    \label{fig:fenced_cliff}
\end{figure*}

Since ``robustness'' is a highly overloaded term in \ac{RL}, we emphasize the notion of robustness in this paper differs from those of previous studies\,\citep{haarnoja2018composable,haarnoja2019learning,eysenbach2022maximum}.
Indeed, they guarantee that entropy regularization preserves a high return under changes in the reward or dynamics.
In other words, the \emph{return} is robust to such changes.
In contrast, we want that \emph{safety constraints are still satisfied} under changes in the dynamics (namely, the level of action noise).
These two types of objectives are complementary since safety and optimality are generally separate concerns in optimal control, where the goal is to act optimally while abiding by safety constraints.
Similarly, the term ``entropy-regularized \ac{RL}'' is used in the literature to refer to various formulations of regularized MDPs \cite{geist2019theory}. In this paper, we use it specifically to refer to methods that optimize the soft \ac{RL} objective \eqref{eq:OCP} as in \citet{HB2019}, which are often referred to as maximum-entropy \ac{RL}. From here on, we reserve the term ``maximum-entropy'' for the special case where the reward is identically zero and the agent solely maximizes entropy.
\par
Our approach builds on two key contributions. First, we empirically show that entropy regularization in a \emph{constrained} environment induces robustness to action noise.
We do this by first showing that agents optimizing the entropy-regularized objective may sacrifice reward to avoid constraints boundaries, with the degree of avoidance modulated by the temperature parameter.
This is illustrated in Figure~\ref{fig:fenced_cliff}, where higher temperatures yield policies whose mode stays farther from constraints.
Then, we provide empirical evidence that this constraints avoidance translates to robustness to action noise, \ie policies preserving the long-term number of viable actions are generally more robust.
\par
This general behavior aligns with the viability-based notion of robustness (called ``safety'' therein) introduced by \citet{HB2019,HRTB2020} (cf.\ \citep{aubin2011viability}), where the robustness of a state is quantified by the number of viable actions --- those that allow indefinite constraint satisfaction. We interpret the cumulative discounted entropy of a policy as a proxy for the long-term number of safe actions it considers, and thus entropy regularization naturally encourages avoidance of states with few viable options. 
\par 
Our second contribution shows that this constrained setting can be approximated arbitrarily well using failure penalties. For penalties above a finite threshold, the \emph{mode} of the resulting policy matches that of the constrained problem, offering a practical reward-shaping strategy for learning robustly safe policies.
\par
Making this observation relevant to state-of-the-art \ac{RL} requires relating the unconstrained and state-constrained optimization problems, as most applications focus on the former \citep{BGHY2022}.
This is the role of \emph{constraints penalization} or, more concisely, of penalties.
In the absence of entropy regularization, they are known to make constraints violations suboptimal, and large enough penalties guarantee that policies optimal for the penalized problem are also optimal for the constrained problem \citep{MHST2022}.
We add entropy regularization to this analysis.
Furthermore, penalties above a \emph{finite} threshold recover the mode of the constrained policy.
\par
Our observations emphasize a benefit of entropy regularization that differs from what is commonly mentioned in the literature.
Indeed, algorithms such as \ac{SAC} \citep{HZA2018} are often praised for their excellent exploration and their robustness to the choice of hyperparameters \citep{HZA2018}.
Although crucial in practice, these strengths are relevant \emph{during} learning.
In contrast, we focus on the optimal policy, that is, on what occurs \emph{after} successful learning.

\paragraph{Contributions}
We reveal how robust optimal controllers arise from the combination of entropy regularization together with sufficient constraints penalties.
Specifically: \begin{compactenum}
    \item We identify empirically that constraints repel trajectories of optimal controllers in the presence of entropy regularization, by favoring controllers maximizing the number of future viable actions.
    \item We prove that failure penalties approximate this constrained problem arbitrarily closely.
    \item Finally, we show that we can extract a safe policy from the optimal solution to the penalized problem and demonstrate that this policy is robustly safe.
\end{compactenum}
The first contribution strongly supports that the mode of entropy-regularized policies is robust to action noise, as the most-likely trajectory is ``repelled'' by the constraints.
We confirm robustness to action noise empirically, and further theoretical investigation is a promising avenue for future work.
Together, our results enable achieving reward-shaping-based robustness, and a novel interpretation of the temperature coefficient in the presence of constraints as a tunable robustness parameter.
\par
The article is organized as follows.
We discuss other approaches to robustness in \ac{RL} in Section~\ref{pos:related work}.
We then expose necessary preliminaries in Section~\ref{pos:preliminaries}, and formalize the problem we consider in Section~\ref{pos:problem formulation}.
Section~\ref{pos:theoretical results} contains our theoretical results, with first a high-level interpretation of the constrained, entropy-regularized problem, and our main theorem guaranteeing approximation with penalties.
The empirical evidence on robustness follows in Section~\ref{pos:empirical results}, together with further empirical validation of our theoretical results.

\section{Related work} \label{pos:related work}

\paragraph{Viability and safety in RL}
There is a variety of definitions of safety in \ac{RL}\,\citep{BGHY2022}.
We consider the case of avoiding state constraints with certainty (level 3 in\,\citep{BGHY2022}).
Such a definition of safety falls into the general problem of viability\,\citep{aubin2011viability}.
Many specialized algorithms were developed to solve this safe \ac{RL} problem, both model-free and model-based\,\citep{achiam2017constrained}.
It has been shown in \citep{MHST2022} that sufficient failure penalties enforce equivalence between the unconstrained and safety-constrained problems, making safe \ac{RL} amenable to unconstrained algorithms.
This idea falls in the class of \emph{penalty methods}, a general idea in optimization which has been studied in the context of optimal control \citep{KM2000,XW1989}; a reformulation of the results of \cite{MHST2022} is that the discounted risk is an exact penalty function.
Our results show that it is no longer the case for entropy-regularized \ac{RL}, as no finite penalty exactly recovers the constrained problem.
Yet, we show it can be approximated arbitrarily closely.
Regardless, the above works only guarantee safety and neglect robustness.
We extend the analysis and proof methods of \citet{MHST2022} to entropy-regularized \ac{RL}, which naturally yields robustness in addition to safety.

\paragraph{Robustness in optimal control}
Robustness is a well-studied topic in optimal control\,\citep{ZDG1996} and consists of preserving viability despite \emph{model uncertainties}.
Classical approaches consist of robust model predictive control\,\citep{GP2017,LAR2009} and Hamilton-Jacobi reachability analysis\,\citep{BCHT2017}.
They provide \emph{worst-case} guarantees, mainly through constraints tightening.
The robustness of entropy-regularized controllers does not fit directly in this category, as their full support makes them explore the whole viability kernel.
Instead, they seem to exhibit a form of ``expected'' constraints tightening, which translates into robustness to action noise of the mode, as we illustrate empirically.
Finally, alternative methods such as scenario optimization\,\citep{calafiore2006scenario} address quantitative uncertainty instead of worst-case, but the connection to the robustness discussed in this article is still open.

\paragraph{Robustness in RL} Achieving robustness for RL policies is an active area of research\,\citep{moos2022robust}. 
A common formalization is that of a two-player game between the agent and an adversary\,\citep{morimoto2005robust,pinto17robust}. 
This setup is akin to that of Hamilton-Jacobi reachability analysis, only with a discounted cost.
These approaches achieve robustness through an adversary controlling, for instance, disturbances \citep{pinto17robust} or action noise \citep{tessler19actionrobust}, yielding worst-case robustness. 
However, such adversarially-robust \ac{RL} requires specialized algorithms and training the adversary. 
In contrast, entropy-regularized \ac{RL} is a popular framework with many standard implementations, which, as we show, also yields robustness solely through reward shaping. 
\par
The work of \citet{HRTB2020} introduces a state-dependent safety measure based on the number of viable actions available in each state. Our work extends this notion to robust safety of policies. A detailed discussion on the connection with the safety measure therein is in Appendix \splittableref{apdx:counterexample off policy metric}{C}.
\par
We are not the first to report that entropy-regularization leads to robustness. 
Some empirical \citep{haarnoja2018composable,haarnoja2019learning} and theoretical works \citep{eysenbach2022maximum} highlight the inherent robustness of entropy-regularized \ac{RL}. 
As mentioned above, however, their definition of robustness differs:
\citet{eysenbach2022maximum} consider robustness of the return to changes in the dynamics, whereas we are interested in preserving constraints satisfaction.

The observation that action noise during training can lead to more robust behavior was already made in \cite[Example 6.6]{SB2018} on the famous cliff walking grid-world. There, $\varepsilon$-greedy action selection resulted in more robust behavior for the case of on-policy learning (SARSA), whereas Q-learning (an off-policy method) learns to the optimal, non-robust, policy. We take the same example in Fig.~\ref{fig:fenced_cliff} and observe that entropy regularization leads to robust behavior in off-policy RL.
Similarly, the $G$-learning algorithm exhibits the same robust behavior on the cliff environment \cite{fox2016taming}. 
Our results and interpretation provide a general explanation for this observation.


\section{Preliminaries} \label{pos:preliminaries}
We introduce concepts to frame the optimization problems and their constraints.
In particular, we address entropy-regularized \ac{RL} and viability.

\subsection{Entropy-regularized RL}
We consider finite sets $\stateSpace$ and $\actionSpace$ called the state and action spaces, respectively, and deterministic dynamics $f:\stateActionSpace\to\stateSpace$, where $\stateActionSpace = \stateSpace\times\actionSpace$ is the state-action space.
A policy $\pi:\stateActionSpace\to[0,1]$ is a map whose partial evaluation in any $x\in\stateSpace$ is a probability mass function on $\actionSpace$; we write $\pi(\cdot\mid x)$, and $\policies$ is the set of all policies.
The state at time $t\in\N$ from initial state $x\in\stateSpace$ and following $\pi\in\policies$ is $X(t;x,\pi)$, and the action taken by $\pi$ at that time is $A(t;x,\pi)$.
If the policy and initial state are unambiguous, we simply write $X_t$ and $A_t$.

We also consider $r:\stateActionSpace\to\R$ a bounded reward function.
The return of $\pi\in\policies$ from initial state $x\in\stateSpace$ is then \begin{equation}\label{eq:cumulative entropy}
    \return(x,\pi) = \sum_{t=0}^\infty \gamma^t r(X_t, A_t),
\end{equation}
where $\gamma\in(0,1)$ is the discount factor.
A smaller $\gamma$ disregards delayed rewards, but can be overcome if the said rewards have large magnitude.
The expected return is $\bar \return(x,\pi) = \E[G(x,\pi)]$.
With $\entropy$ as the entropy, we introduce the discounted cumulative entropy of $\pi\in\policies$ from $x\in\stateSpace$ as \begin{equation}
    \discountedEntropy(x,\pi) = \sum_{t=0}^\infty \gamma^t\entropy(\pi(\cdot\mid X_t)),
\end{equation}
and its expectation $\bar \discountedEntropy(x,\pi) = \E[\discountedEntropy(x,\pi)]$.
The objective of entropy-regularized \ac{RL} is then to find an optimal policy, that is, a policy $\pi_\mathrm{opt}\in\policies$ such that \begin{equation}\label{eq:OCP}
    \pi_\mathrm{opt}\in\arg\max_{\pi\in\policies}\bar \return(x,\pi) + \alpha \bar\discountedEntropy(x,\pi),\quad\forall x\in\stateSpace,
\end{equation}
where $\alpha\in\Rnn$ is a design parameter called the \emph{temperature}.
It is known that there exists an optimal policy \citep{HZA2018}.
Specifically, one can be computed by leveraging the optimal soft-$Q$-value function $q:\stateActionSpace\to\R$, which satisfies for all $(x,a)\in\stateActionSpace$ \citep{NNXS2017}: \begin{equation}
    q(x,a) = r(x,a) + \gamma\alpha\ln\left[\sum_{b\in\actionSpace} \exp\left(\frac1\alpha q(x^\prime, b)\right)\right],
\end{equation}
where we defined the shorthand $x^\prime = f(x,a)$.
An equivalent definition is \citep[Theorem\,16]{NNXS2017} \begin{equation}
    q(x,a) = \max_{\pi\in\policies} r(x,a) + \gamma\bar\return(x^\prime,\pi) + \alpha\gamma \bar \discountedEntropy(x^\prime,\pi).
\end{equation}
Once $q$ is known, the softmax policy solves \eqref{eq:OCP}: \begin{equation}\label{eq:soft optimal policy}
    \pi_\mathrm{opt}(a\mid x) 
        = \softmax\left[\frac1\alpha q(x,\cdot)\right](a) 
        \quad\forall (x,a)\in\stateActionSpace.
\end{equation}
Finally, for any policy $\pi\in\policies$, its \emph{mode} is the policy \begin{equation}\label{eq:mode}
    \hat\pi(a\mid x) = \frac{1}{\lvert\arg\max\pi(\cdot\mid x)\rvert}\delta_{\arg\max\pi(\cdot\mid x)}(a),
\end{equation}
where $\lvert A\rvert$ is the cardinality and $\delta_{A}(a)$ is the indicator function of a set $A\subset\actionSpace$.

\subsection{Viability}
We consider a set of failure states $\constraintSet\subset\stateSpace$ that the system should never visit.
Avoiding $\constraintSet$ is a dynamic concern, and some states that are not in $\constraintSet$ themselves may still lead there inevitably.
We address this through viability theory\,\citep[Chapter~2]{aubin2011viability}.
\begin{definition}[Viability kernel]
    The viability kernel $\viabilityKernel$ is the set of states from where $\constraintSet$ can be avoided at all times almost surely: \begin{equation*}
        \viabilityKernel = \{x\in\stateSpace\mid \exists \pi\in\Pi,\,\forall t\in\Np, \Pbb[X_t\notin\constraintSet] = 1\}.
    \end{equation*}
\end{definition}
By definition, any state that is not in the viability kernel leads to $\constraintSet$ in finite time.
Such states are called \emph{unviable}.
The viability kernel is therefore the largest set that enables recursive feasibility of the problem of avoiding transitions into $\constraintSet$.
A closely related concept is the \emph{viable set}, which is the set of state-action pairs that preserve viability
\citep{HB2019}: \begin{equation*}
    \viableSet = \{(x,a)\in\stateActionSpace\mid x\in\viabilityKernel\land f(x,a) \in\viabilityKernel\}.
\end{equation*}
We also define the unviable set $\unviableSet = \stateActionSpace\setminus\viableSet$, and the critical set $\criticalSet = \unviableSet\cap(\viabilityKernel\times\actionSpace)$ \citep{MHT2021}.

\begin{definition}\label{def:viable policy}
    Let $\pi\in\policies$.
    We say that $\pi$ is \emph{safe from the state $x\in\stateSpace$} if $\Pbb[X_t\notin \constraintSet] = 1$ for all $t\in\Np$.
    We say that $\pi$ is \emph{safe} if it is safe from any $x\in\viabilityKernel$.
    For $\delta>0$, we say that $\pi$ is \emph{$\delta$-safe} if $\max_{\criticalSet}\pi\leq \delta$.
    We denote the set of policies safe from the state $x$ by $\viablePolicies(x)$ and that of safe policies by $\viablePolicies = \bigcap_{x\in\viabilityKernel}\viablePolicies(x)$.
\end{definition}
By definition of the viability kernel, the condition for a safe policy can be replaced with $\Pbb[X_t\in\viabilityKernel] = 1$ for all $t\in\N$.
\begin{remark}
    Another meaningful definition of $\delta$-safety could be that the policy assigns at most $\delta$ of probability mass to unviable actions, that is, $\sum_{a\in\criticalSet[x]}\pi(a\mid x)\leq\delta$ for all $x\in\viabilityKernel$, where $\criticalSet[x]$ is the $\stateSpace$-slice of $\criticalSet$ in $x$.
    This is equivalent to Definition~\ref{def:viable policy} up to the choice of $\delta$, since a $\delta$-safe policy satisfies $\sum_{a\in\criticalSet[x]}\pi(a\mid x)\leq\delta\cdot\lvert\criticalSet[x]\rvert$.
\end{remark}

In the next section, we consider an \ac{RL} problem over the set of safe policies and dual relaxations thereof.
To allow for general such relaxations, we introduce dynamic indicators.
\begin{definition}[Dynamic indicator]\label{def:dynamic indicator}
    Let $c:\stateActionSpace\to\Rnn$ and the associated discounted risk\begin{equation} \label{eq:risk}
        \risk(x,\pi) = \sum_{t=0}^\infty \gamma^t c(X_t, A_t).
    \end{equation}
    We say that $c$ is a dynamic indicator of $\constraintSet$ if, for all $x\in\viabilityKernel$, $\E[\risk(x,\pi)] > 0$ if, and only if, $\pi\notin\viablePolicies(x)$.
\end{definition}
The notion is independent of $\gamma\in(0,1)$. 
A simple example is the composition of the indicator function of $\constraintSet$ with the dynamics; it is a dynamic indicator of $\constraintSet$ \citep[Lemma 1]{MHST2022}.
While this one is always available, more elaborate dynamic indicators help penalize unviable states earlier in the Lagrangian relaxation and lower required penalties, eventually leading to better conditioning.

\begin{remark}[Recovering from constraints violation]\label{rmk:non terminal constraints}
Our results hold in the two settings where visiting $\constraintSet$ terminates the episode or not.
The second case is fully consistent with the setup of infinite time-horizon \ac{RL} that precedes.
Then, actions taken from $\constraintSet$ may map back into $\viabilityKernel$: trajectories leaving $\viabilityKernel$ may only return there \emph{after} visiting $\constraintSet$.
We even have $\constraintSet\cap\viabilityKernel\neq\emptyset$ in general, and the intersection is composed of states with actions that map in $\viabilityKernel\setminus\constraintSet$.
The first case, however, is not naturally framed in infinite time-horizon.
Indeed, while adding an absorbing state with null reward and dynamic indicator as in \citep{MHST2022} effectively cuts the sums in $\return(x,\pi)$ and $\risk(x,\pi)$, the sum in $\discountedEntropy(x,\pi)$ cannot be handled similarly without additional notation.
In the interest of conciseness and clarity, we thus only introduce formally the case of non-terminal $\constraintSet$. 
We emphasize that this is the more challenging case, as forbidding entropy collection after failure effectively further penalizes failure states.
\end{remark}

\section{Problem formulation}\label{pos:problem formulation}

We consider a standard constrained \ac{RL} problem with dynamics $f$, constraint set $\constraintSet$, viability kernel $\viabilityKernel$, return $\return$, and entropy regularization with temperature $\alpha>0$, as defined in Section~\ref{pos:preliminaries}:\begin{align}
    \label{eq:regularized constrained OCP}
    \max_{\pi\in\viablePolicies}~&\bar \return(x,\pi) + \alpha\bar\discountedEntropy(x,\pi).
\end{align}
We investigate the following questions:
\begin{question}\label{q:robustness}
    In what sense can we interpret \eqref{eq:regularized constrained OCP} as a robust control problem?
\end{question}
\begin{question}\label{q:approximation}
    Can we make \eqref{eq:regularized constrained OCP} amenable to unconstrained algorithms?
\end{question}
We provide an empirical answer to Question~\ref{q:robustness} by identifying that the constraints repel trajectories of optimal controllers to an extent controlled by $\alpha$, using tools from viability theory. The higher $\alpha$, the stronger the repulsion.
We then interpret this repulsion as a form of robustness to action noise, as the mode of the solution to \eqref{eq:regularized constrained OCP} favors visiting states where adversarial action noise takes longer to bring the agent to states with constraints.
We support this high-level interpretation with empirical demonstrations on toy examples and standard \ac{RL} benchmarks.
We then answer Question~\ref{q:approximation} through constraints penalties: we show that the solutions of \eqref{eq:regularized constrained OCP} are approximated arbitrarily closely by solving a Lagrangian relaxation of the constraint $\pi\in\viablePolicies$.
Provided that one can solve the resulting unconstrained problem in practice (using for instance classical \ac{RL} algorithms such as \ac{SAC}), our results provide a model-free way to approximate robustly-safe controllers arbitrarily closely with a tunable degree of robustness, as well as a clear interpretation of the temperature and penalty parameters.

\section{Theoretical results} \label{pos:theoretical results}

In this section, we explain on a high level why entropy regularization causes constraints to repel trajectories of optimal controllers and state our theoretical results on how to approximate \eqref{eq:regularized constrained OCP} with a classical unconstrained problem. The proofs are in Appendix~\splittableref{apdx:proofs}{D}. 

\subsection{Preserving future viable options}\label{sec:high level explanation}
\paragraph{Explanation} Our starting point to understand the claimed phenomenon is the observation that, for $x\in\viabilityKernel$, the maximum immediate entropy achievable by a \emph{safe} controller is limited by the number of unsafe actions available in $x$.
Specifically, it follows immediately from properties of $\entropy$ that \begin{equation}\label{eq:immediate entropy bound}
    \forall \pi\in\viablePolicies,~\entropy(\pi(\cdot\mid x)) \leq \ln \lvert\viableSet[x]\rvert.
\end{equation}
Since $\bar\discountedEntropy$ is the (expected discounted) sum of the left-hand side of \eqref{eq:immediate entropy bound} along trajectories, it is meaningful that entropy-regularized, safe optimal controllers avoid states for which this upper bound is low, i.e., where $\lvert\viableSet[x]\rvert$ is low.
On the other hand, completely forbidding actions leading to such states is also harmful, since it ``propagates'' the constraints backwards along trajectories, enforcing a similar upper bound on the immediate entropy obtainable in those previous states as well.
In other words, entropy-regularized controllers limit the probability of actions that eventually lead to states with a low bound in \eqref{eq:immediate entropy bound}, without completely avoiding such actions to avoid loss of immediate entropy.
The more steps it takes to reach states with many constraints, the less pronounced this effect of the constraints is.
It follows from this reasoning that trajectories that go towards states with many constraints generally have lower probability than trajectories that go away from them.
\par
This discussion supports on a high level that entropy regularization with constraints promotes constraints avoidance by preserving the long-term number of future viable options.
Next, we identify this behavior as a form of robustness to action noise of the mode policy. Indeed, the mode policy tends to minimize the long-term proportion of actions unavailable because of constraints, and thus the probability that action noise selects such an action is also approximately minimized.
We leave a precise formalization of this idea to future work, and support it with empirical evidence in Section \ref{pos:empirical results}.

\paragraph{A metric of robustness}
This discussion highlights that, for any $\pi\in\viablePolicies$ and $x\in\viabilityKernel$, the quantity $\bar\discountedEntropy(x,\pi)$ captures the \emph{long-term number of viable actions} that $\pi$ considers from $x$.
A controller achieving a high $\bar\discountedEntropy(x,\pi)$ successfully avoids highly-constrained states.
This motivates taking the cumulative entropy as a quantitative measurement of robustness, which enables comparing the robustness of controllers.
\begin{definition}\label{def:s-robustness}
    We say that $\pi_1\in\viablePolicies$ is \emph{less $\robSymbol$-robust} than $\pi_2\in\viablePolicies$, and write $\pi_1\leqRobust\pi_2$, if \begin{equation}\label{eq:s-robustness}
        \bar S(x,\pi_1) \leq \bar S(x,\pi_2),\quad \forall x\in\viabilityKernel.
    \end{equation}
\end{definition}

\paragraph{Behavior for increasing temperatures}

For $\alpha=0$, \eqref{eq:regularized constrained OCP} recovers the constrained, unregularized problem \begin{equation}
     \label{eq:constrained OCP}
    \max_{\pi\in\viablePolicies}~\bar \return(x,\pi).
\end{equation} 
We are then maximizing the return over viable policies with no concerns about robustness.
As $\alpha$ increases, entropy is more and more prevalent in the objective of \eqref{eq:regularized constrained OCP}, whose solution converges to the maximum entropy policy $\maxEntPolicy$\begin{equation}\label{eq:max ent OCP}
    \maxEntPolicy = \arg\max_{\pi\in\viablePolicies} \bar\discountedEntropy(x,\pi),\quad\forall x\in\viabilityKernel.
\end{equation}
This is best seen through the soft-value function.
\begin{theorem}\label{thm:constrained increasing alpha}
    Consider the soft-Q-value functions $\maxEntQValue$ and $\constrainedQValue{\alpha}$ of \eqref{eq:max ent OCP} and \eqref{eq:regularized constrained OCP}, respectively and for all $\alpha\in\Rnn$.
    Then, $\max_{\viableSet}\rvert\frac1\alpha\constrainedQValue{\alpha}-\maxEntQValue\lvert\to0$ as $\alpha\to\infty$.
\end{theorem}

\begin{corollary}\label{clry:robustness}
    Denote by $\maxEntPolicy$ and $\constrainedPolicy{\alpha}$ the solutions of \eqref{eq:max ent OCP} and \eqref{eq:regularized constrained OCP}, respectively and for all $\alpha\in\Rnn$.
    Then, the map $\alpha\mapsto\constrainedPolicy{\alpha}$ is monotonic for $\leqRobust$ and $ \max_{\viableSet}\lvert\constrainedPolicy{\alpha}-\maxEntPolicy\rvert\to 0$ as $\alpha\to\infty$.
\end{corollary}

Corollary~\ref{clry:robustness} formalizes that the solution of \eqref{eq:regularized constrained OCP} becomes more $\robSymbol$-robust and approaches that of \eqref{eq:max ent OCP} as $\alpha$ increases.
The mode of \eqref{eq:regularized constrained OCP} thus gets robust to action noise by the preceding explanations and empirical evidence of Section~\ref{pos:empirical results}.

\subsection{Relaxing safety constraints with penalties}
\label{ssec:penalized}

A practical consequence of our observation is that solving \eqref{eq:regularized constrained OCP} yields controllers that preserve a safe distance to the constraints with high probability.
An important drawback, however, is that the problem involves the viable set $\viableSet$, which is unknown in model-free situations.
We now leverage a Lagrangian relaxation of these viability constraints to make the problem amenable to model-free algorithms.
The results in this section extend those of \cite{MHST2022} to the case of an entropy-regularized objective.

In this section, we consider $c$ a dynamic indicator function of $\constraintSet$ and $\risk$ the associated discounted risk (Definition~\ref{def:dynamic indicator}).
We are interested in the following penalized problem \begin{equation}\label{eq:penalized OCP}
    \penalizedPolicy{\alpha}{p} = \arg\max_{\pi\in\policies}\bar\return(x,\pi) + \alpha\bar\discountedEntropy(x,\pi) - p\risk(x,\pi),
\end{equation}
where $p\in\Rnn$ is a penalty parameter.
It is known that in the case $\alpha=0$, \eqref{eq:penalized OCP} and \eqref{eq:regularized constrained OCP} share the same solutions if $p$ is large enough \citep[Theorem\,2]{MHST2022}.
Unfortunately, this result does not directly carry to the case $\alpha>0$: from \eqref{eq:soft optimal policy}, $\penalizedPolicy\alpha p(a\mid x)>0$ for all $(x,a)\in\stateActionSpace$, and thus in particular $\penalizedPolicy{\alpha}{p}\notin\viablePolicies$.
However, scaling the penalty remains possible if one accepts to trade viability for $\delta$-safety.
\begin{theorem}\label{thm:approximation RCOCP}
    For any $\delta>0$, $\epsilon>0$, and $\alpha>0$, there exists $\optimalPenalty\in\Rnn$ such that, for all $p>\optimalPenalty$, the optimal policy of \eqref{eq:penalized OCP} $\penalizedPolicy{\alpha}{p}$ is $\delta$-safe and \begin{equation}
        \max_{\viableSet}\lvert\penalizedPolicy{\alpha}{p} - \constrainedPolicy\alpha\rvert < \epsilon.
        \end{equation}
\end{theorem}
\begin{proof}[Sketch of proof]
    The penalty enforces an upper-bound on the soft $Q$-value of state-actions in $\criticalSet$ (Lemma~\ref{lemma:bounds penalized q value}).
    Values there thus decrease arbitrarily low as the penalty increases, while it remains lower-bounded on $\viableSet$.
    This, in turn, shows $\delta$-safety of $\penalizedPolicy\alpha p$ for $p$ large enough.
    Therefore, the value function of \eqref{eq:penalized OCP} approximates to that of \eqref{eq:regularized constrained OCP} on $\viableSet$, and $\penalizedPolicy\alpha p$ gets arbitrarily close to $\constrainedPolicy\alpha$.
\end{proof}

\subsection{Safe policies from the relaxed problem}
\label{ssec:mode}

It follows from Theorem~\ref{thm:approximation RCOCP} that the solution $\penalizedPolicy{\alpha}{p}$ to the penalized problem \eqref{eq:penalized OCP} is a $\delta$-safe policy and the \emph{mode} of $\penalizedPolicy{\alpha}{p}$ is safe if the penalty is sufficient.
\begin{corollary}\label{co:safemode}
    Under the same notations as Theorem~\ref{thm:approximation RCOCP}, there exists $\bar\delta\in(0,1)$ such that, if $\delta\in(0,\bar\delta)$, then the policy $\hat{\pi}_{\alpha,p}$ following the mode \eqref{eq:mode} of $\penalizedPolicy{\alpha}{p}$
    is safe.
\end{corollary}
\begin{proof}
    This directly follows from Theorem~\ref{thm:approximation RCOCP}.
\end{proof}
We empirically investigate the robustness of this policy in the next section.

\paragraph{Conclusion on the questions}
Finally, we are able to answer Questions~\ref{q:robustness} and~\ref{q:approximation} based on the following arguments.
Entropy regularization in the presence of constraints biases the learning problem towards policies that avoid constraints to preserve a high number of viable options, with the temperature coefficient monotonically controlling the degree of $\robSymbol$-robustness.
Furthermore, the viability constraints of \eqref{eq:regularized constrained OCP} can be relaxed by a Lagrangian formulation at the price of trading viability for $\delta$-safety.
Specifically, the solution of \eqref{eq:penalized OCP} approximates arbitrarily closely that of \eqref{eq:regularized constrained OCP}, provided that the penalty is sufficiently high.
In particular, penalties above a finite threshold recover the mode of \eqref{eq:regularized constrained OCP} exactly and the policy following that mode is therefore safe.
Put together, these results provide a model-free way to approximate safe and robust controllers with tunable degrees of robustness.

\section{Empirical results} \label{pos:empirical results}

We provide in this section the empirical evidence that entropy regularization with constraints yields policies whose mode avoids constraints and is robustly safe under increased action noise, and that penalties enable approximating these constraints.
We start with a discrete grid world, where we can solve the constrained problem numerically, to showcase how constraints repel trajectories in the presence of entropy regularization.
Second, we introduce failure penalties to reveal how they recover the constraints.
Finally, we illustrate the claimed robustness to increased action noise on MuJoCo benchmarks\footnote{The code to reproduce results is available at \url{www.github.com/Data-Science-in-Mechanical-Engineering/entropy_robustness}.}.
These experimental results confirm our interpretation of the two hyperparameters: penalties control the probability of failure, while the temperature controls the degree of robustness.

\subsection{Cliff walking}\label{ssec:cliff}

Our gridworld (Fig.~\ref{fig:fenced_cliff}) is an adaptation of the cliff environment\,\citep[Example~6.6]{SB2018}.
Three states in the middle of the bottom row represent the cliff; the failure set $\constraintSet$ the agent should robustly avoid.
The right column represents the target of escaping the cliff.
The failure and target states are invariant under all actions.
Otherwise, the dynamics follow the direction of the chosen action, or map back into the current state if the agent hits a border.
Actions outside of the cliff or target get a $-1$ reward.

\subsubsection{Interaction of constraints and entropy}

\begin{figure}[t]
    \centering
    \includegraphics{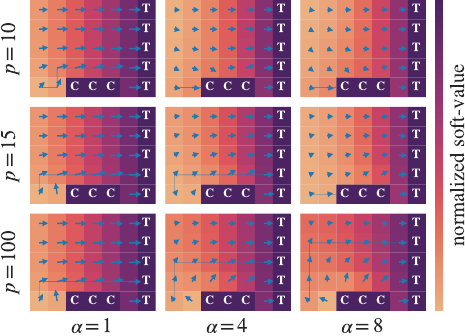}
    \caption{\textbf{Unconstrained cliff --- Safety and robustness as functions of $\alpha$ and $p$:} Safety and robustness can be achieved by penalizing ($p$) the constraints $\constraintSet$ and adjusting the temperature ($\alpha$).}
    \label{fig:robust_cliff}
\end{figure}

The constrained version of the environment --- the fenced cliff --- only offers three actions to an agent neighboring the cliff, imposing a lower upper-bound on the entropy in those states as per \eqref{eq:immediate entropy bound}.
This observation is key in understanding why entropy regularization avoids states with unviable actions, yielding robustness (Fig. \ref{fig:fenced_cliff}.d).

Indeed, when maximizing entropy only (Fig.~\ref{fig:fenced_cliff}.d), the optimal policy favors transitioning away from states neighboring the constraints due to the aforementioned upper bound on immediate entropy.
In turn, the immediate entropy of the policy in the 2-step neighbors is also reduced since some transitions are less desirable.
The same logic applies recursively ``outwards'' from states with unviable actions, and the policy generally pushes trajectory away from the constraints; that is, towards the top corners.
When initialized on the right, the policy aims at reaching the invariant target states where full entropy is available.
When initialized on the left, the mode of the policy favors the top-left corner to avoid the low entropy of states close to the constraints, overcoming the long-term benefit of the target state.
This trade-off between short- and long-term entropy depends on the discount factor $\gamma$.

In contrast, finite temperatures (Fig.~\ref{fig:fenced_cliff}.a--c) further encourage reaching the goal state to avoid the negative reward.
The agent thus takes more risks to collect rewards while preserving some distance from the constraints.
This trade-off between performance and robustness is controlled by the temperature parameter $\alpha$:
high values favor entropy (and, thus, robustness by what precedes), whereas lower ones favor performance.
While high robustness may be desirable, it comes at the price of suboptimality.
Too high a temperature may entirely prevent task completion for the mode policy if the path thereto is inherently risky, leading to unsuccessful learning outcomes due to poor choice of hyperparameters.

\subsubsection{Interaction of penalties and entropy}\label{ssec:unconstrained_cliff}

Sufficient penalties enable solving the constrained problem (Fig.~\ref{fig:robust_cliff}), consistently with Theorem~\ref{thm:approximation RCOCP}.
The example shows the robustness--performance trade-off with different temperatures and penalties. 
Importantly, entropy and penalties are now competing, and any fixed penalty is eventually overcome by high temperatures, degrading safety (Fig. \ref{fig:delta_safe}).
The penalty thus needs to scale with the temperature to ensure $\delta$-safety with a low $\delta$.
\par
The minimum sufficient value for the penalty depends not only on $\alpha$, but also on other hyperparameters such as the reward function, discount factor, and dynamic indicator.
For instance, if the dynamic indicator is simply the indicator function of the constraints set, then the minimum penalty scales exponentially with the longest trajectory contained in $\stateSpace\setminus(\viabilityKernel\cup\constraintSet)$.
Other choices of dynamic indicators may improve this dependency by incurring the penalty earlier in the trajectory, but choosing the penalty remains a problem-specific concern.
While theory suggests picking it as high as possible, too high a penalty may introduce numerical instabilities when combined with value function approximators outside of tabular methods.
We refer to \cite{MHST2022} for an extended discussion.

\begin{figure}[t]
    \centering
    \resizearxiv{
    \includegraphics{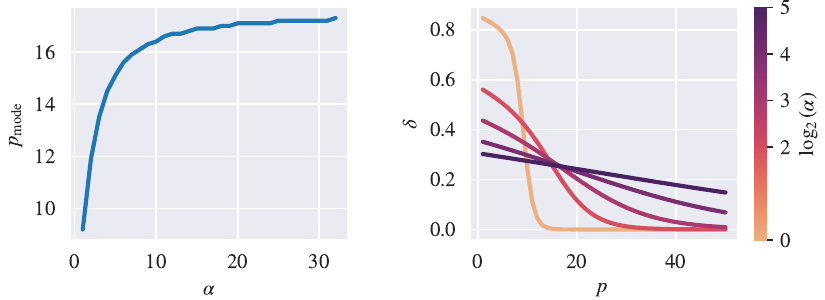}
    }
    \caption{\textbf{Effect of the temperature on the minimum safe penalty $p_{\mathrm{mode}}$ (left) and $\delta$-safety (right) on the cliff:} \emph{Left:} The minimum penalty such that the mode of the stochastic policy is safe.
    \emph{Right:} The minimum $\delta$ such that the policy is $\delta$-safe as functions of $p$ and $\alpha$. Policies get safer as $p$ increases, but less safe as $\alpha$ does. }
    \label{fig:delta_safe}
\end{figure}

\subsection{Reinforcement learning benchmarks}
\label{ssec:rl}

We now illustrate on standard \ac{RL} benchmarks that this constraints avoidance translates into increased robustness to action noise.
For this, we train entropy-regularized agents on two popular MuJoCo benchmarks, namely the \texttt{Pendulum-v1} and the \texttt{Hopper-v4} environment \citep{towers2023gymnasium}, with various temperatures.
We then evaluate the mode of the learned policy under additional external action noise, whereas training is noise-free.
The action noise is sampled from a uniform distribution $\mathcal{U}(-\epsilon, \epsilon)$.
For each value of the temperature, we evaluate the frequency of successful constraints avoidance over $100$ episodes.
Further details on the setup are in Appendix~\splittableref{apdx:experiment_details}{B} and additional results are in Appendix~\splittableref{apdx:cliff}{A}.
\par
Consistently with our theoretical results, we find that (i) entropy-regularization decreases the return by avoiding high-value states with many unviable actions; and (ii) the mode of entropy-regularized policies is more robust to disturbances as the training temperature increases.

\subsubsection{Pendulum}

\begin{figure*}
    \centering
    \raisebox{-0.5\height}{%
        \input{figures/tikz/pendulum_blue}%
        \input{figures/tikz/pendulum_green}%
        \input{figures/tikz/pendulum_red}%
    }%
    \\
    \vspace{2mm}
    \raisebox{-0.5\height}{\includegraphics{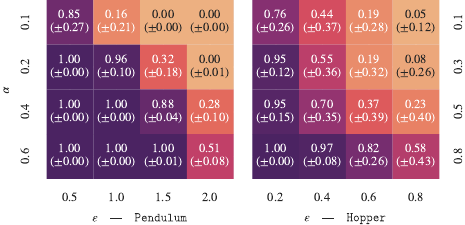}}
    \caption{\textbf{Learning robust policies with SAC} 
    \emph{Top:} With a target angle at $40^{\circ}$ (dashed-dotted line) the agent learns to stabilize at different angles depending on the training temperature. For higher temperatures, the agent stabilizes the pendulum further away from the failure set $\constraintSet$.
    \emph{Bottom:} 
    Rate of successful failure avoidance on the disturbed \texttt{Pendulum-v1} (left heat map) and \texttt{Hopper-v4} (right heat map) environments. As the temperature increases the mode of the stochastic policy is robust to higher levels of action noise $\epsilon$.}
    \label{fig:pendulum}
\end{figure*}

We modify the \texttt{Pendulum-v1} environment as follows to incorporate robustness concerns:
(i) the initial state is the still, upright position; (ii) the constraints consist of angles with magnitude beyond $90\degree$ and the penalty is $90$; and (iii) the reward is the squared angular difference to a target angle of $40 \degree$, which is outside of the viability kernel since the agent exerts bounded torque.

The results are shown in Fig.~\ref{fig:pendulum}. All policies lean towards the target state but avoid leaving the viability kernel and reaching the constraints. 
The sufficient penalty emulates the boundary of the viability kernel, which reduces the effective number of available actions when leaning to one side.
This pushes entropy-regularized policies away from the target state, and they learn to stabilize angles closer to $0$ as $\alpha$ increases --- the maximum entropy policy keeps the pendulum upright.
The results show a robustness--performance trade-off between staying upright and leaning as far as possible towards the target, which is controlled by the temperature $\alpha$.
Furthermore, the mode of the entropy-regularized policy can cope with significantly higher action disturbances when trained with higher temperatures.

\subsubsection{Hopper}

We repeat the same experiment as in the previous section for a modified \texttt{Hopper-v4} environment \citep{towers2023gymnasium}. We modify the environment by penalizing the ``unhealthy'' states with a penalty of $p=500$. The results are shown in Fig.~\ref{fig:pendulum}. Increasing the temperature improves the robustness to additional action noise. However, the learned gait is slower, hinting at a performance--robustness trade-off for this environment (see Appendix~\splittableref{apdx:result_rl}{A.2}). Interestingly, as the temperature is increased, the training finds two distinct robust behaviors: one is the intended hopping forward; the other is standing still and only collecting the healthy reward, which is arguably the most robust behaviour.

Our experiments inform hyperparameter settings for \ac{RL} practitioners: while entropy regularization leads to robustly safe policies, high temperatures (or minimum entropy constraints\,\citep{HZHT2018}) can make parts of the state space unreachable, lead to conservative policies, and may even entirely prevent task completion as seen in the Hopper example.

\section{Conclusion}

We study the interaction between entropy regularization and state constraints in \ac{RL} and reveal empirically that this favors policies that are constraints-avoiding and robust to increased action noise, as they preserve an expected long-term number of viable actions.
We also show both in theory and in practice how to approximate the constraints with failure penalties.
In particular, the mode of the policy --- which is often what is deployed after training completion --- is recovered exactly by penalties above a finite threshold.
\par
The connection between entropy regularization with constraints and control-theoretic robustness is novel, to the best of our knowledge.
This study identifies the phenomenon, its relevance for \ac{RL}, and opens many interesting avenues for future work.
A particularly promising one is the systematic study of the identified robustness.
Indeed, we hypothesize that entropy regularization with constraints induces a kind of \emph{soft constraints tightening}; that is, restricts the optimization domain to controllers that go away from the constraints with at least some given probability.
Such a result would enable identifying ``softly invariant sets'': subsets of the viability kernel that are control invariant under a robustly safe controller (but not directly under entropy-regularized controllers, as they have full action support) and contain the entropy-regularized controller's trajectory with high probability.
This would draw a clear theoretical bridge between entropy-regularized constrained \ac{RL} and robust control through constraints tightening.
An alternative would be identifying a noise model to which the modes of entropy-regularized, constrained policies are robust.
More generally, it would be interesting to find other regularization terms that promote robustness and that are amenable to \ac{RL} beyond the cumulative entropy, following ideas from \cite{geist2019theory}. Such regularizers could enable novel robustness properties with rigorous guarantees, and perhaps help training policies that are less sensitive to the sim-to-real gap.
\par
In the meantime, we expect our findings to inform practitioners when applying \ac{RL} algorithms such as \ac{SAC}.
While entropy regularization has mainly been developed as an exploration mechanism\,\citep{HZA2018}, it biases the policy to robustness to action noise if one uses constraints penalties. 
This understanding enables principled decisions when tuning the temperature and penalties, for instance by discouraging the common practice of annealing the temperature if robustness is a concern.

\ifarxiv
\subsubsection*{Acknowledgments}
\else
\begin{credits}
\subsubsection{\ackname} 
\fi
The authors thank Zeheng Gong for help with the empirical results. Simulations were performed with computing resources granted by RWTH Aachen University under project rwth1626.
This work has been supported by the Robotics Institute Germany, funded by BMBF grant 16ME0997K.

\ifarxiv
\subsubsection*{Disclosure of interests}
\fi
The authors have no competing interests to declare.
\ifarxiv\else
\end{credits}
\fi

%
%
%
%
\ifarxiv
\bibliographystyle{unsrtnat}
\fi
\bibliography{references.bib}


\ifarxiv

\appendix
\clearpage
\input{appendix/extended_results}
\clearpage
\input{appendix/experimental_details}

\clearpage
\input{appendix/counterexample}
\clearpage
\input{appendix/proofs}

\fi
\end{document}

%% file: macros.tex
\newcommand{\stateSpace}{\mathcal{X}}

\newcommand{\constraintSet}{\stateSpace_\mathrm{C}}
\newcommand{\viabilityKernel}{\stateSpace_\mathrm{V}}

\newcommand{\actionSpace}{\mathcal{A}}

\newcommand{\stateActionSpace}{\mathcal{Q}}
\newcommand{\viableSet}{\stateActionSpace_\mathrm{V}}
\newcommand{\unviableSet}{\stateActionSpace_\mathrm{U}}
\newcommand{\criticalSet}{\stateActionSpace_\mathrm{crit}}
\newcommand{\policies}{\Pi}
\newcommand{\viablePolicies}{\Pi_\mathrm{V}}

\newcommand{\return}{G}
\newcommand{\discountedEntropy}{S}
\newcommand{\entropy}{\mathcal{H}}
\newcommand{\risk}{\rho}
\newcommand{\maxEntQValue}{Q_{\mathrm{ent}}}
\newcommand{\constrainedQValue}[1]{Q_{#1}}
\newcommand{\penalizedQValue}[2]{Q_{#1,#2}}
\newcommand{\maxEntPolicy}{\pi^\star_{\mathrm{ent}}}
\newcommand{\constrainedPolicy}[1]{\pi^\star_{#1}}
\newcommand{\penalizedPolicy}[2]{\pi^\star_{#1,#2}}
\newcommand{\optimalPenalty}{p^\star}

\newcommand{\robSymbol}{\discountedEntropy}
\newcommand{\leqRobust}{\preceq}
\newcommand{\tc}{{T_\mathrm{C}}}
\newcommand{\tr}{{T_\mathrm{R}}}

\newcommand{\R}{\mathbb{R}}
\newcommand{\Rp}{\mathbb{R}_{>0}}
\newcommand{\Rnn}{\mathbb{R}_{\geq0}}
\newcommand{\N}{\mathbb{N}}
\newcommand{\Np}{\mathbb{N}_{>0}}
\newcommand{\E}{\mathbb{E}}
\newcommand{\Pbb}{\mathbb{P}}

\newcommand{\goesto}[1]{\xrightarrow[#1]{}}

\newcommand{\softmax}{\mathrm{softmax}}

\newacro{RL}[RL]{reinforcement learning}
\newacro{SAC}[SAC]{soft actor-critic}
\newacro{LHS}[LHS]{left-hand side}

\newcommand{\ie}{i\/.\/e\/.,\/~}

%% file: figures/tikz/pendulum_blue.tex
\begin{tikzpicture}

\coordinate (pivot) at (0,0);
\coordinate (bob) at ({1.25*cos(104.9)},{1.25*sin(104.9)}); 

\coordinate (target) at ({1.5*cos(130)},{1.5*sin(130)});
\draw[dash dot] (pivot) -- (target) node[anchor=north] {};

\filldraw [black] (pivot) circle (2pt);

\draw[line width=5pt, blue, opacity=0.5,line cap=round] (pivot) -- (bob);

\draw[->] (0,1.) arc [start angle=90, end angle=104.9, radius=1.];


\node at (0, 1.8) {$\alpha=0.1$};

\draw[dashed] (pivot) -- (0,1.5) node[anchor=south] {};

\fill[pattern=north west lines, pattern color=red] (-1.25,0) rectangle (1.25,-.5);
\draw[] (-1.25,0) -- (1.25,0) node[anchor=north] {$\constraintSet$};


\draw[dashed] (1.25,0) arc [start angle=0, end angle=180, radius=1.25];

\end{tikzpicture}

%% file: figures/tikz/pendulum_green.tex
\begin{tikzpicture}

\coordinate (pivot) at (0,0);
\coordinate (bob) at ({1.25*cos(97.3)},{1.25*sin(97.3)}); 

\coordinate (target) at ({1.5*cos(130)},{1.5*sin(130)});
\draw[dash dot] (pivot) -- (target) node[anchor=north] {};

\filldraw [black] (pivot) circle (2pt);

\draw[line width=5pt, green, opacity=0.5,line cap=round] (pivot) -- (bob);

\draw[->] (0,1.) arc [start angle=90, end angle=97.3, radius=1.];


\node at (0, 1.8) {$\alpha=0.3$};

\draw[dashed] (pivot) -- (0,1.5) node[anchor=south] {};

\fill[pattern=north west lines, pattern color=red] (-1.25,0) rectangle (1.25,-.5);
\draw[] (-1.25,0) -- (1.25,0) node[anchor=north] {$\constraintSet$};


\draw[dashed] (1.25,0) arc [start angle=0, end angle=180, radius=1.25];

\end{tikzpicture}

%% file: figures/tikz/pendulum_red.tex
\begin{tikzpicture}

\coordinate (pivot) at (0,0);
\coordinate (bob) at ({1.25*cos(90.7)},{1.25*sin(90.7)}); 

\coordinate (target) at ({1.5*cos(130)},{1.5*sin(130)});
\draw[dash dot] (pivot) -- (target) node[anchor=north] {};

\filldraw [black] (pivot) circle (2pt);

\draw[line width=5pt, red, opacity=0.5,line cap=round] (pivot) -- (bob);

\draw[->] (0,1.) arc [start angle=90, end angle=90.7, radius=1.];


\node at (0, 1.8) {$\alpha=1$};

\draw[dashed] (pivot) -- (0,1.5) node[anchor=south] {};

\fill[pattern=north west lines, pattern color=red] (-1.25,0) rectangle (1.25,-.5);
\draw[] (-1.25,0) -- (1.25,0) node[anchor=north] {$\constraintSet$};


\draw[dashed] (1.25,0) arc [start angle=0, end angle=180, radius=1.25];

\end{tikzpicture}

%% file: appendix/extended_results.tex
\section{Extended empirical results}\label{apdx:cliff}

We provide some additional empirical results to support our findings in Section~\ref{pos:empirical results}.

\subsection{Cliff walking}

\paragraph{Policy entropy} In Fig.~\ref{fig:max_ent}, we show the immediate entropy of the optimal policy in each state for both the maximum entropy policy and the regularized policy.
Fig.~\ref{fig:max_ent}.a shows the effect of constraints on the maximum entropy policy, pushing away trajectories. In the regularized setting (Fig.~\ref{fig:max_ent}.b) the optimal policy trades entropy for return.

\begin{figure}[!h]
    \centering
    \includegraphics{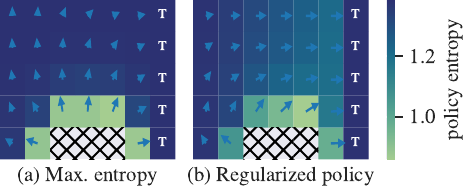}
    \caption{\textbf{Immediate entropy of the maximum entropy policy (a) and of the regularized policy (b):} To avoid states with unviable actions (and thus low maximum entropy as per \eqref{eq:immediate entropy bound}), $\maxEntPolicy$ moves away from the constraint (a). The regularized policy ($\alpha=4$) trades entropy for return and stays closer to the constraints (b).
    The colormap is the entropy of the optimal policy and the blue arrows the expected actions. Shorter arrows thus mean a higher action distribution entropy.}
    \label{fig:max_ent}
\end{figure}

\clearpage
\subsection{RL benchmarks}\label{apdx:result_rl}
\paragraph{Trade-off between robustness and return} In Fig.~\ref{fig:returns} we show for different values of $\alpha$ the total return of the trained policies in the undisturbed pendulum and hopper environments. The result shows that returns decrease as $\alpha$ increases, that is, as policies become more robust.

For the pendulum, the return is directly correlated to the angle at which it stabilizes.
The decrease in return reflects a more upright stabilized position (cf.~Fig~\ref{fig:pendulum}).
Given the finite torque the agent can apply to the pendulum, the upright equilibrium is the most robust viable position in this environment and, indeed, for higher temperatures the agent stabilizes the pendulum closer to its equilibrium. 

A similar trend can be observed for the hopper environment.
The overall return decreases as the agent learns a smaller but more robust gait. 
For higher temperatures, we observe two different learning outcomes, depending on the randomness of the learning process. 
One where the agent still learns to walk slowly forward and another where the agents stands still.
This example already indicates that training with high entropy and penalties can lead to undesired outcomes where the agent cannot complete its task anymore since doing so would be too risky.

\begin{figure}[!h]
    \centering
    \includegraphics{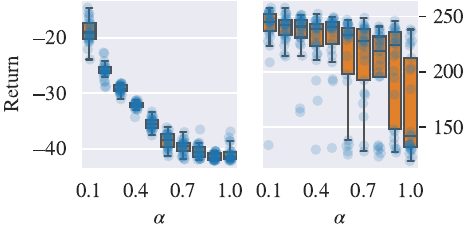}
    \caption{\textbf{Effect of the temperature on return:} Evaluation on the \texttt{Pendulum-v1} (left) and \texttt{Hopper-v4} (right) environments. As the temperature increases, the mode of the policy trades return for robustness and gets more conservative.}
    \label{fig:returns}
\end{figure}

\paragraph{Success rate under disturbances} We show in Fig.~\ref{fig:full_disturbance_heatmaps_pendulum} and Fig.~\ref{fig:full_disturbance_heatmaps_hopper} extended versions of the heat maps shown in Fig.~\ref{fig:pendulum}. 
They show the dependency between the training temperature $\alpha$ and the maximum magnitude of disturbances $\epsilon$ the agent can withstand at evaluation.
Generally, as the training temperature increases, so does the robustness of the mode of the learned policy.
For the hopper environment (Fig.~\ref{fig:full_disturbance_heatmaps_hopper}) the trend is intact but our results exhibit high variance over different training runs, especially for high temperatures. We attribute this to suboptimal solutions found during learning. When inspecting the trained policy's behavior we see that sometimes the agent learns to stand still in order to increase robustness.

\begin{figure}[!ht]
    \centering
    \resizebox{\textwidth}{!}{
    \includegraphics{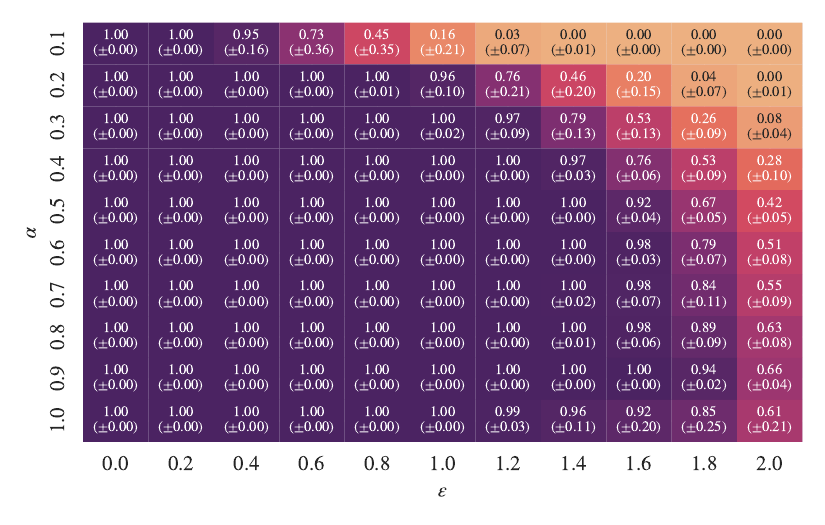}
    }
    \caption{\textbf{Robustness to disturbances for the pendulum environment:} Evaluation of a policy trained with different temperatures $\alpha$. We show the mean and standard deviation of the success rate evaluated over training with $25$ random seeds. As the magnitude $\epsilon$ of the disturbance is increased, the mode policy retains higher success rates for higher temperatures.}
    \label{fig:full_disturbance_heatmaps_pendulum}
\end{figure}

\begin{figure*}[!ht]
    \centering
    \resizebox{\textwidth}{!}{
    \includegraphics{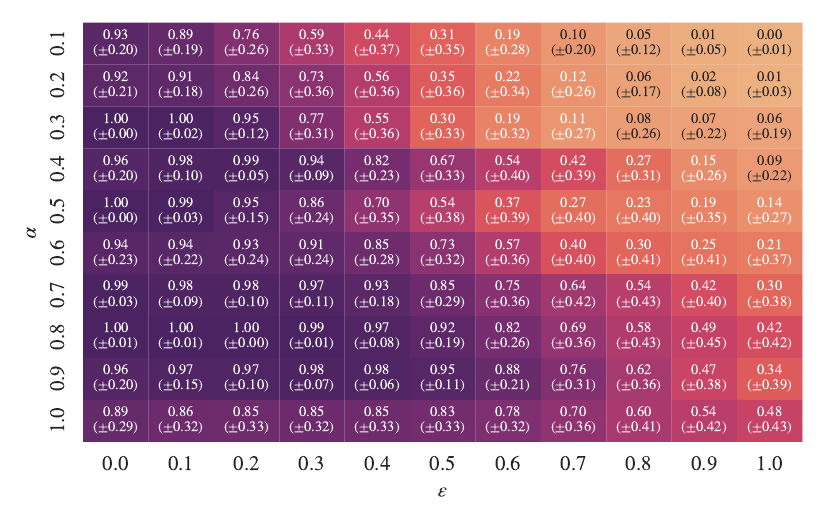}}
    \caption{\textbf{Robustness to disturbances for the hopper environment:} Evaluation of a policy trained with different temperatures $\alpha$. We show the mean and standard deviation of the success rate evaluated over training with $25$ random seeds. As the magnitude $\epsilon$ of the disturbance is increased, the mode policy retains higher success rates for higher temperatures.}
\label{fig:full_disturbance_heatmaps_hopper}
\end{figure*}

%% file: appendix/experimental_details.tex
\section{Experimental details}\label{apdx:experiment_details}
All experiments in Sec.~\ref{ssec:rl} were conducted on a compute cluster. In total, around 50 core-hours on an Intel Xeon 8468 Sapphire and 450 core-hours on an NVIDIA H100 have been used to produce the results.
The code will be published upon acceptance and is also part of the supplementary material of the submission.

The hyperparameters for the tabular value iteration to compute the policies in Sec.~\ref{ssec:cliff} are given in Table~\ref{tab:hyperparameters-tabular}.

We train the neural network policies in Sec.~\ref{ssec:rl} using the \ac{SAC} implementation of \citet{huang2022cleanrl} with different fixed temperature values (no automatic tuning). The hyperparameters are reported in Table~\ref{tab:hyperparameters-deep-rl}.
For each $\alpha$, we train on $25$ random seeds and evaluate the trained policy in terms of return and robustness. 
Training curves are shown in Fig.~\ref{fig:training}.
To evaluate robustness, we take the mode of each trained policy and add a disturbance sampled from a uniform distribution $\mathcal{U}(-\epsilon, \epsilon)$ to the action at each environment step. An episode is successful if the agent abides by the constraint despite the input disturbance. We test each trained policy for $100$ episodes and report the success rate and return.

\subsection{Environments}
\paragraph{Fenced Cliff}

The Fenced Cliff environment is illustrated in Fig.~\ref{fig:fenced_cliff} and Fig.~\ref{fig:max_ent}. The agent starts in the bottom-left corner with the goal of reaching one of the target states in the right column (marked as \textbf{T}). In general, the agent can take one of four actions to move to adjacent states (excluding diagonal moves) or remain in place if the action would lead out of bounds. However, actions that would lead in the failure set are not available. Consequently, states bordering the cliff have a smaller action space than all other states. In particular, the agent cannot enter the cliff. The agent receives a reward of $-1$ (plus the entropy reward) for every action outside a target state. Reaching a target state does not end the episode. Instead, the agent keeps collecting entropy rewards in target states and remains in its current state for all four actions.

\paragraph{Unconstrained Cliff}

The Unconstrained Cliff environment is illustrated in Fig.~\ref{fig:robust_cliff}. It relaxes the Fenced Cliff environment by allowing the agent to enter the cliff, that is, the failure set (marked as \textbf{C}). Transitioning into the cliff from a non-failure state results in a penalty of $-p$ where $p$ varies between experiments. Similar to target states, entering the cliff does not end the episode but traps the agent there indefinitely, as all four actions in the cliff map back to it. In particular, the agent keeps collecting entropy rewards inside failure states.

\paragraph{Pendulum}

We use the standard \texttt{Pendulum-v1} environment \citep{towers2023gymnasium} with a modified reward function
\begin{equation}
    r(x, a) = -(\theta - \theta_{\mathrm{target}})^2,
\end{equation}
that incentivizes the agent to move towards the edge of the viability kernel, as illustrated in Fig. \ref{fig:pendulum_sketch}. The agent starts in the upright position and applies control inputs in $[-2; 2]$ exercising a torque $\tau$ that effectively controls the angle $\theta$. We use a target angle of $\theta_\mathrm{target} = 40\degree$ relative to the upright position. The failure set comprises those states with an angle of at least $90\degree$ and entering the failure set results in a penalty of $90$. The still, upright position is therefore the most robust state as it maintains the largest distance to the constraint set. Episodes terminate after 200 steps or when entering the failure set.

\begin{figure}[!h]
    \centering
    \input{figures/tikz/pendulum}
    \caption{A sketch of the pendulum environment with target angle $\theta_{\mathrm{target}}$ and constraint set $\constraintSet$.}
    \label{fig:pendulum_sketch}
\end{figure}

\paragraph{Hopper}

We use the standard \texttt{Hopper-v4} MuJoCo environment \citep{towers2023gymnasium,todorov2012mujoco}. This environment has a built-in notion of the agent being \emph{healthy}, defined as fulfilling a set of constraints under which the agent maintains a human-like upright position. The environment thus admits a natural failure set, namely the set of states in which the agent is not healthy. In our modification of the environment, entering the failure set results in a penalty of $500$. Other than that, we use the default reward function which incentivizes moving forward and maintaining a healthy state while avoiding large control inputs. To that end, the agent applies control inputs in $[-1; 1]^3$ exerting torques on its three joints, thereby controlling movement and posture. Episodes terminate after 1000 steps or when entering the failure set.

\subsection{Implementation details}

\paragraph{Model selection for SAC}

For our study on the robustness of the trained policies in the Pendulum and Hopper environments, we choose the policy that achieved the highest objective value during training, that is, the weighted sum of returns and entropy.

\paragraph{Early termination for SAC}

In contrast to the cliff walking experiments, we terminate an episode whenever we enter the failure set. We do this because this is how the standard benchmarks are implemented. However, early termination affects the entropy regularized value function since the agent cannot collect any more entropy in the failure set and therefore reduces the discounted reward. Early termination is therefore equivalent to an additional penalty and reduces $\optimalPenalty$ where this reduction automatically scales with the temperature.

\begin{table*}[!hb]
  \caption{Hyperparameters for tabular RL experiments on Fenced Cliff and Unconstrained Cliff.}
  \label{tab:hyperparameters-tabular}
  \centering
  \begin{tabularx}{\textwidth}{>{\raggedright\arraybackslash}m{0.4\textwidth} >{\raggedright\arraybackslash}m{0.6\textwidth}}
    \toprule
    \textbf{Parameter} & \textbf{Value}  \tabularnewline
    \midrule
    Discount factor & $0.95$ \tabularnewline
    Maximum number of iterations & $1,000$ \tabularnewline
    Convergence tolerance & $0.00001$ \tabularnewline
    \bottomrule
  \end{tabularx}
\end{table*}

\begin{table*}[!hb]
  \caption{Hyperparameters for deep RL experiments on Pendulum and Hopper.}
  \label{tab:hyperparameters-deep-rl}
  \centering
  \begin{tabularx}{\textwidth}{>{\raggedright\arraybackslash}m{0.4\textwidth} >{\raggedright\arraybackslash}m{0.3\textwidth} >{\raggedright\arraybackslash}m{0.3\textwidth}}
    \toprule
    \textbf{Parameter} & \textbf{Value (Pendulum)} & \textbf{Value (Hopper)}  \tabularnewline
    \midrule
    Penalty & $90$ & $300$ \tabularnewline
    Total steps & $250,000$ & $1,000,000$ \tabularnewline
    Buffer size & $250,000$ & $1,000,000$ \tabularnewline
    Batch size & $256$ & $256$ \tabularnewline
    Discount factor & $0.99$ & $0.99$ \tabularnewline
    Target smoothing coefficient & $0.005$ & $0.005$ \tabularnewline
    Policy learning rate & $0.0003$ & $0.0003$ \tabularnewline
    Q learning rate & $0.001$ & $0.001$ \tabularnewline
    Optimizer & Adam & Adam \tabularnewline
    Policy update frequency & $2^{-1}$ & $2^{-1}$ \tabularnewline
    Target network update frequency & $1$ & $1$ \tabularnewline
    \bottomrule
  \end{tabularx}
\end{table*}

\begin{figure*}[!ht]
    \centering
    \resizebox{\textwidth}{!}{
    \includegraphics{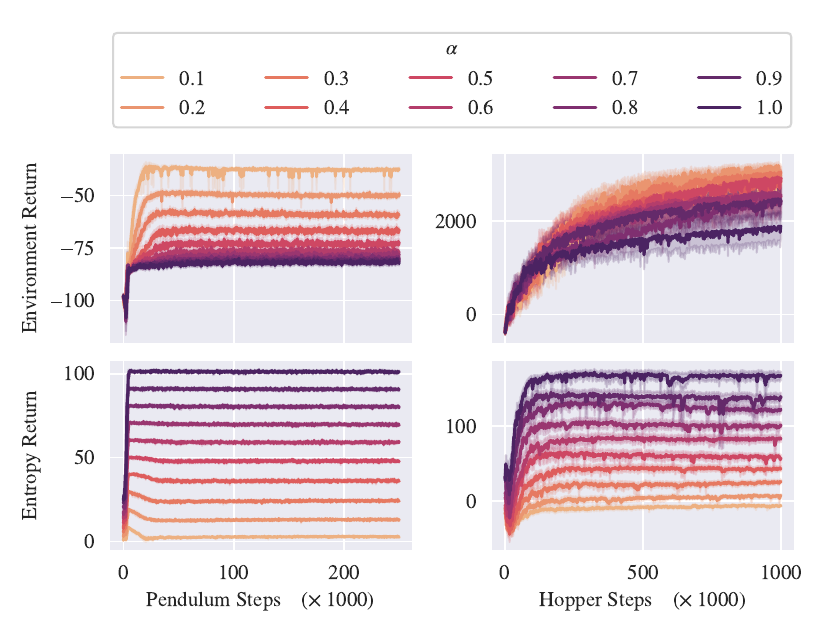}
    }
    \caption{Mean undiscounted environment return (top row) and discounted entropy return (bottom row) during training in the \texttt{Pendulum-v1} (left column) and \texttt{Hopper-v4} (right column) environments with bootstrapped $95\%$ confidence interval over 25 random seeds for each temperature $\alpha$.}
    \label{fig:training}
\end{figure*}

%% file: figures/tikz/pendulum.tex
\begin{tikzpicture}

\coordinate (pivot) at (0,0);
\coordinate (bob) at ({1.25*cos(110)},{1.25*sin(110)}); 

\coordinate (target) at ({1.5*cos(135)},{1.5*sin(135)});
\draw[dashed] (pivot) -- (target) node[anchor=north] {$\theta_{\mathrm{target}}$};

\filldraw [black] (pivot) circle (2pt);

\draw[line width=5pt, blue, opacity=0.5,line cap=round] (pivot) -- (bob);

\draw[->] (0,1.) arc [start angle=90, end angle=110, radius=1.];

\node at (-0.2, 1.2) {$\theta$};

\draw[dashed] (pivot) -- (0,1.5) node[anchor=south] {};

\fill[pattern=north west lines, pattern color=red] (-1.25,0) rectangle (1.25,-.5);
\draw[] (-1.25,0) -- (1.25,0) node[anchor=north] {$\constraintSet$};

\draw[->, very thick] (0,0.3) arc [start angle=90, end angle=420, radius=0.3];
\node at (0.4, 0.4) {$\tau$};

\draw[dashed] (1.25,0) arc [start angle=0, end angle=180, radius=1.25];

\end{tikzpicture}

%% file: appendix/counterexample.tex
\section{Relation to the off-policy safety measure of \texorpdfstring{\citet{HRTB2020}}{Heim et al. (2020)}}\label{apdx:counterexample off policy metric}

In this section we discuss potential problems when using, instead of the cumulative discounted entropy $S$ in Definition~\ref{def:s-robustness}, the cumulative discounted safety measure, which is introduced in \citet{HRTB2020}.
This safety measure defines the ``natural robustness of a state $x\in\viabilityKernel$'' as the number of viable actions available in that state, that is, the cardinality of $\viableSet[x]$.
\par
A first observation is that the safety measure of a state, $\Lambda(x)$, is closely connected to the entropy of the uniform viable policy $u_V$ through $\entropy(u_V(\cdot\mid x) = \ln(\Lambda(x))$.
In other words, using the (logarithm of the) safety measure as a reward means maximizing the quantity \begin{equation*}
    \bar T(x,\pi) = \E_\pi\left[\sum_{t=0}^\infty \entropy(u_V(\cdot\mid X_t))\right],
\end{equation*}
which differs from $\bar \discountedEntropy$ defined in \eqref{eq:cumulative entropy} by replacing the policy used to compute the entropy by $u_V$.
It might --- at first sight --- seem like a meaningful quantification of robustness, since it directly measures the ``long term number of viable actions'' available by following a policy.
\par
Unfortunately, this quantification is not satisfactory.
Indeed, it allows, for instance, visiting states where most viable actions lead to only having only few viable options remaining, as long as at least one path guarantees many choices.
A compelling example illustrating why a policy with a high such $\bar T$ should not be called robust is available below.
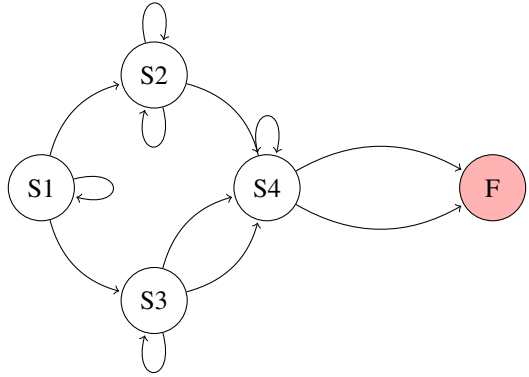
\begin{figure}[!h]
    \centering
    \begin{tikzpicture}[shorten >=1pt, node distance=1.5cm and 3cm, on grid, auto]
       \node[state] (s1) at (0, 0)   {S1};
       \node[state] (s2) at (1.5, 1.5) {S2};
       \node[state] (s3) at (1.5, -1.5){S3};
       \node[state] (s5) at (3, 0) {S4};
       \node[state, fill=red!30] (s4) [right=of s5] {F};
    
       \path[->]
        (s1) edge[bend left, above] node {} (s2)
             edge[bend right, below] node {} (s3)
             edge[loop right] node {} (s5)
        (s2) edge[bend left, above] node {} (s5)
             edge[loop above] node {} (s2)
             edge[loop below] node {} (s2)
        (s3) 
             edge[bend right=30] node {} (s5)
             edge[bend left=30] node {} (s5)
             edge[loop below] node {} (s3)
        (s5) edge[bend left=30] node {} (s4)
             edge[bend right=30] node {} (s4)
             edge[loop above] node {} (s5)
             ;
    \end{tikzpicture}
    \caption{\textbf{Counterexample for an off-policy metric of robustness based on entropy.}
    The constraint set is $\constraintSet = \{F\}$.}
    \label{fig:counterexample}
    \vspace{-3mm}
\end{figure}
\begin{example}
Consider the Markov decision process of Fig~\ref{fig:counterexample}, and two policies: $\pi_1$, that goes from $S1$ to $S2$, and $\pi_2$, that goes from $S1$ to $S3$.
Arguably, in the absence of a noise model, a meaningful notion of robustness should say that $\pi_1$ is more robust than $\pi_2$; indeed, in the presence of action perturbations, both may end up in $S4$, but $\pi_1$ takes a path with more ``chances'' to avoid --- or delay --- going there.
Yet, the value of $\bar T(S_1,\pi_2)$ is the same as that of $\bar T(S_1,\pi_1)$, since the former is increased by the additional action from $S3$ to $S4$.
Choosing this metric thus classifies $\pi_2$ as equally robust as $\pi_1$, which is the not the desired outcome.
\end{example}

The property of the above metric allowing this example is that the policy only needs to visit states with a high number of viable actions, without considering whether those actions preserve a high number of viable actions themselves, as long as at least one does.
In other words, the policy does not need to commit to the viable actions that contribute to the value of its metric, and the evaluation of the robustness via $\entropy(u_V(\cdot\mid X_t))$ is myopic.
There are two simple solutions to this issue.
The first is to use an \emph{on-policy} metric of robustness, which is what we do in the present work by replacing $u_V$ in the previous proposition by the policy itself, yielding $\bar \discountedEntropy$.
Then, there are policies that visit $S2$ with a metric higher than the value achieved by any policy that visits $S3$.
The second option is to use a quantity that is already sensitive to the long-term robustness of a state instead of $\entropy(u_V(\cdot\mid X_t))$, such as for instance $I(x) = \E_{u_V}[\sum_{t=0}^\infty \gamma^t \entropy\{u_V(\cdot\mid X(t;x,u_V))\}]$, and to define robustness as having a high value of $\bar U(x,\pi) = \E_\pi\left[\sum_{t=0}^\infty \gamma^t I(X_t)\right]$.
This raises the problem, however, that computing this $I$ involves solving a reinforcement learning problem itself.

%% file: appendix/proofs.tex
\section{Proofs}\label{apdx:proofs}
In all of this section, we identify real functions defined on $\viableSet$ with vectors in $\R^{|\viableSet|}$, and use operators defined on the functions on their vector representations indistinctly.

\subsection{Proof of Theorem~\ref{thm:constrained increasing alpha}}
\begin{proof}
    For $\alpha\in\Rp$, consider the operators $B^\star$ and $B^\star_\alpha$ defined on $\R^{\viableSet}$ by, for all $h\in\R^{\viableSet}$ and $(x,a)\in\viableSet$, \begin{align*}
        (B^\star h)(x,a) &= \gamma\ln\left[\sum_{b\in\viableSet[x]} \exp[h(x^\prime,b)]\right],\\
        (B^\star_\alpha h)(x,a) &=\frac1\alpha r(x,a) + \gamma\ln\left[\sum_{b\in\viableSet[x]} \exp[h(x^\prime,b)]\right],
    \end{align*}
    where we introduced the shorthand $x^\prime = f(x,a)$.
    The operator $B^\star$ is the optimal soft-Bellman operator associated with \eqref{eq:max ent OCP}, and $B^\star_\alpha$ is that associated with \eqref{eq:constrained OCP}, scaled by the factor $\frac1\alpha$.
    Let now $R_\alpha = \frac1\alpha \constrainedQValue{\alpha}$.
    By properties of the optimal Bellman operator, $R_\alpha$ and $\maxEntQValue$ are the unique fixed points of $B^\star_\alpha$ and $B^\star$, respectively.
    
    Let us fix a strictly increasing sequence $(\alpha_n)_{n\in\N}$, and define for conciseness $R_{\alpha_n} = R_n$ and $B^\star_{\alpha_n} = B^\star_n$.
    The sequence $(R_n)_{n\in\N}$ is clearly bounded.
    Therefore, there exists $R\in\R^{|\viableSet|}$ and $(m_n)_{n\in\N}$, strictly increasing, such that $R_{m_n}\to R$ as $n\to\infty$.
    We show $B^\star (R) = R$.
    For all $n\in\N$,\begin{align*}
        \lVert R - B^\star R\rVert 
            &\leq \lVert R - R_{m_n}\rVert + \lVert B^\star_{m_n}(R_{m_n}) - B^\star (R_{m_n})\rVert + \lVert B^\star (R_{m_n}) - B^\star (R)\rVert\\
            &\leq \lVert R - R_{m_n}\rVert + \frac1{\alpha_n} \lVert r\rVert +  \lVert B^\star (R_{m_n}) - B^\star (R)\rVert,
    \end{align*}
    where we used the identity $B^\star_{m_n}(R_{m_n}) = R_{m_n}$ in the first inequality.
    By continuity of $B^\star$, we have $B^\star (R_{m_n})\to B^\star (R)$, and thus the right-hand side converges to $0$ as $n\to\infty$.
    We deduce that $R = B^\star (R)$.
    But then, by uniqueness of the fixed point of $B^\star$, $R = \maxEntQValue$.
    This shows that the bounded sequence $(R_n)_{n\in\N}$ admits a unique accumulation point $\maxEntQValue$, and thus converges to $\maxEntQValue$.
    We have thus shown that for any strictly increasing sequence $(\alpha_n)_{n\in\N}$, $R_{\alpha_n}\to\maxEntQValue$ as $n\to\infty$.
    This concludes the proof by the sequential characterization of convergence of a function.
\end{proof}

\subsection{Proof of Corollary~\ref{clry:robustness}}
\begin{proof}
    We focus on monotonicity since convergence follows immediately from Theorem~\ref{thm:constrained increasing alpha} and \eqref{eq:soft optimal policy}.
    Let $\alpha$ and $\beta$ be in $\Rnn$, with $\alpha\leq\beta$.
    Let $x\in\viabilityKernel$.
    By definition of $\constrainedPolicy{\alpha}$ and $\constrainedPolicy{\beta}$, we have both \begin{align*}
        \bar\return(x,\constrainedPolicy{\alpha}) + \alpha \bar\discountedEntropy(x,\constrainedPolicy{\alpha}) &\geq \bar\return(x,\constrainedPolicy{\beta}) + \alpha \bar\discountedEntropy(x,\constrainedPolicy{\beta}),\\
        \bar\return(x,\constrainedPolicy{\alpha}) + \beta\bar\discountedEntropy(x,\constrainedPolicy{\alpha}) &\leq \bar\return(x,\constrainedPolicy{\beta}) + \beta \bar\discountedEntropy(x,\constrainedPolicy{\beta}).
    \end{align*}
    Taking the difference and rearranging yields \begin{equation*}
        (\alpha-\beta) (\bar\discountedEntropy(x,\constrainedPolicy{\alpha}) - \bar\discountedEntropy(x,\constrainedPolicy{\beta})) \geq 0.
    \end{equation*}
    We deduce that the two factors have the same sign, i.e., $\bar\discountedEntropy(x,\constrainedPolicy{\alpha}) \leq \bar\discountedEntropy(x,\constrainedPolicy{\alpha})$.
    Since this is valid for all $x\in\viabilityKernel$, this concludes the proof.
\end{proof}

\subsection{Proof of Theorem~\ref{thm:approximation RCOCP}}
We begin with two preliminary technical results on dynamic indicators.
For any trajectory $\tau = [(x_n,a_n)]_{n\in\N}\subset\stateActionSpace^\N$, we introduce the notations \begin{align*}
    \tc(\tau) &= \min\{t\in\N\mid x_t\in\constraintSet\},\\
    \tr(\tau) &= \min\{t\in\N\mid x_t\in\constraintSet\land x_{t+1}\notin\constraintSet\},
\end{align*}
with the convention $\min\emptyset = \infty$.
\begin{lemma}[Characterization of dynamic indicators]\label{lemma:characterization dynamic indicator}
    Let $c:\stateActionSpace\to\Rnn$.
    Then, $c$ is dynamic indicator of $\constraintSet$ if, and only if, $c_{\mid \viableSet} = 0$ and for any trajectory $\tau = [(x_t,a_t)]_{t\in\N}\subset\stateActionSpace^\N$ such that $x_0\in\viabilityKernel$ and $\tc(\tau)<\infty$, there exists $t\leq \tr(\tau)$ such that $c(x_t, a_t) > 0$.
\end{lemma}
\begin{proof}
    We first show the converse implication.
    Let $x\in\viabilityKernel$ and $\pi\in\policies$.
    If $\pi\notin\viablePolicies(x)$, then let $\tau=[(x_n,a_n)]_{n\in\N}$ be a trajectory starting from $x$ and generated by $\pi$ with positive probability such $\tc(\tau)<\infty$.
    Let $t\leq\tr(\tau)$ such that $c(x_t,a_t)>0$.
    Then, \begin{align*}
        \bar\risk(x,\pi) 
            \geq L\cdot\sum_{u=0}^\infty \gamma^u c(x_u,a_u)
            \geq L\cdot \gamma^t c(x_t,a_t)
            > 0,
    \end{align*}
    where we defined $L=\Pbb[((X(t;x,\pi),A(t;x,\pi)))_{t\in\{0,\dots,T\}} = ((x_t,a_t))_{t\in\{0,\dots,T\}}]>0$ for conciseness.
    Conversely, if $\bar \risk(x,\pi) > 0$, then there must exist a trajectory $\tau=[(x_t,a_t)]_{t\in\N}$ starting from $x$ and with positive probability such that $\sum_{t=0}^\infty \gamma^t c(x_t,a_t) > 0$.
    But there must then be a $t\in\N$ such that $(x_t,a_t)\notin\viableSet$, by assumption on $c$.
    By definition, the trajectory $\tau$ reaches $\constraintSet$ in finite time after that time $t$.
    Since $\tau$ has positive probability by following $\pi$, this shows that $\pi\notin\viablePolicies(x)$ and shows the implication.

    For the converse implication, assume that $c$ is a dynamic indicator of $\constraintSet$.
    Let $(x,a)\in\viableSet$, and take any policy $\pi\in\viablePolicies$ such that $\pi(a\mid x) = 1$ (such a policy exists by definition of the viability kernel).
    Then, $0 = \bar\risk(x,\pi)\geq c(x,a)$, and thus $c_{\mid\viableSet} = 0$.
    Next, let $\tau= [(x_t,a_t)]_{t\in\N}$ be such that $x_0\in\viabilityKernel$ and $\tc(\tau)<\infty$.
    Introduce $T_\mathrm{V} = \max\{t\in\N\mid t\leq\tc(\tau)\land (x_t,a_t)\in\viableSet\}$, and define $(x,a)=(x_{T_\mathrm V+1},a_{T_\mathrm V+1})$.
    For every $y\in\stateSpace$, define the set of actions that $\tau$ takes in state $y$,\begin{equation*}
        \tau(y) = \{b\in\actionSpace\mid \exists t\in\N, (x_t,a_t) = (y,b)\}.
    \end{equation*}
    Let $\theta:\stateSpace\to\actionSpace$ be a map such that
    \begin{enumerate}
        \item $\theta(x) = a$;
        \item for all $y\in\viabilityKernel$ with $y\neq x$, $\theta(y)\in\viableSet[y]$;
        \item for all $y\notin\viabilityKernel$ with $\tau(y)\neq\emptyset$, $\theta(y)\in\tau(y)$.
    \end{enumerate}
    Define the policy $\pi(b\mid y) = \delta_{\theta(y)}(b)$.
    Clearly, $\pi\notin\viablePolicies(x)$, since it takes action $a$ in $x$, and thus $\bar \risk(x,\pi)>0$.
    Define now $\tau_{0:\mathrm R} = \{(y,b)\in\stateActionSpace\mid \exists t\in\{0,\dots,\tr(\tau)\}, (y,b) = (x_t,a_t)\}$.
    Crucially, $\pi$ only explores state-action pairs that are either in $\viableSet$ or in $\tau_{0:\mathrm R}$ when initialized in $\viabilityKernel$.
    We thus obtain (almost-surely)\begin{align*}
        \bar\risk(x,\pi) 
            &= \sum_{t=0}^\infty \gamma^tc(X(t;x,\pi),A(t;x,\pi))\\
            &= \sum_{t=0}^\infty \sum_{(y,b)\in\tau_{0:\mathrm R}}\gamma^tc(y,b)\cdot\delta_y(X(t;x,\pi))\cdot\delta_b(A(t;x,\pi))\\
            &= \sum_{(y,b)\in\tau_{0:\mathrm R}}c(y,b)\cdot\sum_{t=0}^\infty\gamma^t\cdot\delta_y(X(t;x,\pi))\cdot\delta_b(A(t;x,\pi)),
    \end{align*}
    where we leveraged the fact that $c_{\mid\viableSet} = 0$.
    Therefore, there must be $(y,b)\in\tau_{0:\mathrm R}$ such that $c(y,b) > 0$, concluding the proof.
\end{proof}
\begin{corollary}\label{clry:dynamic indicator on unviable set}
    Let $c$ be a dynamic indicator of $\constraintSet$, $\risk$ the associated discounted risk \eqref{eq:risk}, and $q\in\criticalSet$.
    For any $q\in\criticalSet$, $c(q) + \gamma\min_{\pi\in\policies}\bar \risk(f(q),\pi)>0$.
\end{corollary}
\begin{proof}
    From classical results on dynamic programming, there exists $\eta\in\policies$ such that $\bar\risk(f(q),\eta) = \min_{\pi\in\policies}\bar\risk(f(q),\pi)$.
    Furthermore, $\eta$ can be chosen deterministic.
    Consider then the trajectory of $\eta$ starting from $f(q)$, and prepend $q$ to it.
    Let us call $\tau$ the resulting trajectory.
    It begins in $\viabilityKernel$ and $\tc(\tau)<\infty$, since $q\in\criticalSet$.
    Therefore, by Lemma~\ref{lemma:characterization dynamic indicator}, there exists $t\leq\tr(\tau)$ such that $c(x_t,a_t)>0$, where $\tau = [(x_u,a_u)]_{u\in\N}$.
    The result follows from $c(q) + \bar\risk(f(q),\eta) \geq \gamma^t c(x_t,a_t)>0$.
\end{proof}
The following lemma is the core of the proof of Theorem~\ref{thm:approximation RCOCP}: it upper-bounds the soft-$Q$-value on $\criticalSet$, and lower-bounds it on $\viableSet$. It is a generalization to the present setting of \citep[Lemma\,2]{MHST2022}.
\begin{lemma}\label{lemma:bounds penalized q value}
    There exists constants $v_1, u_1, u_2\in\R$ and $u_3\in\Rp$ such that, for all $\alpha\in\Rp$ and $p\in\Rnn$, \begin{align}
        \min_{\viableSet}\penalizedQValue{\alpha}{p} &\geq v_1, \label{eq:lower bound viable set}\\
        \max_{\criticalSet}\penalizedQValue{\alpha}{p} &\leq u_1 + \alpha\cdot u_2 - p\cdot u_3. \label{eq:upper bound unviable set}
    \end{align}
\end{lemma}
\begin{proof}
    By \citep[Theorem\,1]{HZHT2018}, we have in particular for all $\alpha\in\Rnn$ and $p\in\Rnn$ that $\penalizedQValue{\alpha}{p}\geq Q^\pi_{\alpha,p}$, where $Q^\pi_{\alpha,p}$ is the soft $Q$-value function of an arbitrary policy $\pi\in\policies$, that is, the only fixed point of the operator $B^\pi_{\alpha,p}$ defined for all $h\in\R^\stateActionSpace$ as \begin{equation*}
        (B^\pi_{\alpha,p}h)(x,a) = r(x,a) + \gamma \E[h(x^\prime, A_1) - \ln \pi(A_1\mid x^\prime)],
    \end{equation*}
    for all $(x,a)\in\stateActionSpace$ and where we defined the shorthand $x^\prime = f(x,a)$.
    Take for $\pi$ a deterministic policy.
    Then, $B^\pi_{\alpha,p}$ simplifies to the classical Bellman operator, and thus the soft-$Q$-value function coincides with the classical $Q$-value function $Q^\pi_p$.
    Therefore, $\penalizedQValue{\alpha}{p} \geq Q^\pi_{p} $ for any $\pi\in\policies$ deterministic.
    This holds in particular if $\pi\in\viablePolicies$, for which \begin{align*}
        Q^\pi_p(x,a) &= r(x,a) - p\cdot c(x,a) + \gamma (\return(x^\prime,\pi) - p\risk(x^\prime,\pi))\\
            &= r(x,a) + \gamma\return(x^\prime,\pi),
    \end{align*}
    for all $(x,a)\in\viableSet$ and with $x^\prime = f(x,a)$.
    Indeed, $c(x,a) = 0$ by Lemma~\ref{lemma:characterization dynamic indicator} since $(x,a)\in\viableSet$, and $\risk(x^\prime,\pi) = 0$ since $\pi\in\viablePolicies(x^\prime)$.
    Since the right-hand side is lower-bounded (by boundedness of $r$) by a constant independent of $p$ and $\alpha$, we deduce the existence of $v_1$ as announced.

    In contrast, from \citep[Theorem\,16]{NNXS2017}, it follows that for all $(x,a)\in\criticalSet$ \begin{align*}
        \penalizedQValue{\alpha}{p}(x,a) 
            &= \max_{\pi\in\policies} r(x,a) + \gamma \bar \return(x^\prime,\pi) + \alpha\gamma \discountedEntropy(x^\prime,\pi) - p c(x,a) - p\gamma\bar\risk(x^\prime,\pi)\\
            &\leq \max_{\pi\in\policies} \frac{1}{1-\gamma} (\sup_\stateActionSpace r + \alpha\gamma\lvert \actionSpace\rvert \ln\lvert\actionSpace\rvert) - p c(x,a) - p \gamma\bar\risk(x^\prime,\pi)\\
            &= u_1 + \alpha\cdot u_2  - p\cdot\left(c(x,a) + \gamma\min_{\pi\in\policies}\bar\risk(x^\prime,\pi)\right),
    \end{align*}
    with $x^\prime = f(x,a)$.
    Here, we have defined $u_1 = \frac{1}{1-\gamma} \sup_\stateActionSpace r$ and $u_2 = \frac{1}{1-\gamma} \lvert \actionSpace\rvert \ln\lvert\actionSpace\rvert$.
    Furthermore, let $u_3 = \min_{q\in\criticalSet}\left[c(q) + \min_{\pi\in\policies} \bar \risk(f(q),\pi)\right]$, which is positive by Corollary~\ref{clry:dynamic indicator on unviable set} and since $\criticalSet$ is finite.
    This yields the desired upper bound for all $(x,a)\in\criticalSet$.
    Since the constants are independent of $\alpha$ and $p$, this concludes the proof.
\end{proof}
This enables showing convergence of the unconstrained, penalized soft-$Q$-value to its constrained counterpart.
\begin{lemma}\label{lemma:convergence q value}
    For all $\alpha>0$, we have \begin{align*}
        \sup_{\criticalSet}\penalizedQValue\alpha p&\goesto{p\to\infty} - \infty,\\
        \sup_{\viableSet}\lvert\penalizedQValue\alpha p - \constrainedQValue\alpha\rvert&\goesto{p\to\infty}0.
    \end{align*}
\end{lemma}
\begin{proof}
    The first claim follows immediately from Lemma~\ref{lemma:bounds penalized q value}.
    For the second one, recall from \citet[Theorem\,16]{NNXS2017} that for all $(x,a)\in\stateActionSpace$ \begin{equation*}
        \penalizedQValue{\alpha}{p}(x,a) 
            = \max_{\pi\in\policies} r(x,a) + \gamma \bar \return(x^\prime,\pi) + \alpha\gamma \discountedEntropy(x^\prime,\pi) - p c(x,a) - p\gamma\bar\risk(x^\prime,\pi).
    \end{equation*}
    It follows that the function $p\mapsto \penalizedQValue\alpha p(x,a)$ is nonincreasing.
    If $(x,a)\in\viableSet$, it is also lower-bounded by Lemma~\ref{lemma:bounds penalized q value}.
    Therefore, there exists $\bar Q:\viableSet\to\R$ such that $\lim_{p\to\infty}\penalizedQValue\alpha p = \bar Q$, pointwise on $\viableSet$ (and, thus, uniformly as well).
    We show $\bar Q = \constrainedQValue\alpha$.

    Let $(x,a)\in\viableSet$ and $x^\prime = f(x,a)$.
    For any $p\in\Rnn$, the definition of $\penalizedQValue\alpha p$ as the unique fixed point\,\citep{NNXS2017} of the soft-Bellman operator associated to the reward $r - pc$ gives \begin{equation*}
        \penalizedQValue\alpha p(x,a) = r(x,a) + \alpha\gamma\ln\left[\sum_{b\in\actionSpace}\exp\left[\frac1\alpha\penalizedQValue\alpha p(x^\prime,b)\right]\right],
    \end{equation*}
    recalling that $c(x,a) = 0$.
    Taking the limit as $p\to\infty$ on both sides yields \begin{align*}
        \bar Q(x,a) 
            &= r(x,a) + \alpha\gamma\ln\left[\sum_{b\in\actionSpace}\exp\left[\frac1\alpha\lim_{p\to\infty}\penalizedQValue\alpha p(x^\prime,b)\right]\right]\\
            &= r(x,a) + \alpha\gamma\ln\left[\sum_{b\in\viableSet[x^\prime]}\exp\left[\frac1\alpha\bar Q(x^\prime,b)\right] + \sum_{b\in\criticalSet[x^\prime]}0\right]\\
            &= r(x,a) + \alpha\gamma\ln\left[\sum_{b\in\viableSet[x^\prime]}\exp\left[\frac1\alpha\bar Q(x^\prime,b)\right]\right],
    \end{align*}
    where we used the fact that $\lim_{p\to\infty}\penalizedQValue\alpha p(x,b)=-\infty$ for all $b\in\criticalSet[x^\prime]$.
    In other words, $\bar Q$ is a fixed point of the same Bellman operator as $\constrainedQValue\alpha$. Since that fixed point is unique and equal to $\constrainedQValue\alpha$, this concludes the proof.
\end{proof}
\begin{proof}[Proof of Theorem~\ref{thm:approximation RCOCP}]
    Let $\delta, \epsilon,$ and $\alpha$ be in $\Rp$.
    Let $p_1 = \frac{u_1-v_1}{u_3}+\frac{u_2 -\ln\delta}{u_3}\alpha$ and $p>p_1$.
    By Lemma~\ref{lemma:bounds penalized q value}, we have \begin{align*}
        \penalizedPolicy{\alpha}{p}(a\mid x) 
            &= \frac{
                \exp\left[\frac1\alpha \penalizedQValue{\alpha}{p}(x,a)\right]
            }{
                \sum_{b\in\actionSpace} \exp\left[\frac1\alpha Q(x,b)\right]
            }\\
            &\leq \exp\left[\frac{u_1-v_1}{\alpha} + u_2 - \frac{u_3}\alpha p\right]\\
            &< \delta,
    \end{align*}
    for all $(x,a)\in\criticalSet$.
    
    Next, it follows from Lemma~\ref{lemma:convergence q value} that $\exp\left[\frac1\alpha\penalizedQValue\alpha p(x,a)\right]$ converges to $\exp\left[\frac1\alpha\constrainedQValue\alpha(x,a)\right]$ if $(x,a)\in\viableSet$, and to $0$ if $(x,a)\in\criticalSet$.
    It follows immediately that there exists $p_2\in\Rnn$ such that, for all $(x,a)\in\viableSet$ and $p>p_2$, \begin{equation*}
        \left\lvert\softmax\left[\frac1\alpha \penalizedQValue\alpha p(x,\cdot)\right](a) - \softmax\left[\frac1\alpha \constrainedQValue\alpha(x,\cdot)\right](a)\right\rvert<\epsilon.
    \end{equation*}
    The result follows by taking $\optimalPenalty = \max\{p_1,p_2\}$.
\end{proof}